\documentclass{article}

 \usepackage[preprint]{neurips_2026}

\usepackage[utf8]{inputenc} 
\usepackage[T1]{fontenc}    
\usepackage{hyperref}       
\usepackage{url}            
\usepackage{booktabs}       
\usepackage{amsfonts}       
\usepackage{nicefrac}       
\usepackage{microtype}      
\usepackage{xcolor}  
\usepackage{multirow}
\usepackage{enumitem}
\usepackage{amsmath,amssymb,amsthm}
\usepackage{graphicx}
\usepackage{wrapfig}
\usepackage{subcaption}
\usepackage{array}
\usepackage{tabularx}

\newcolumntype{Y}{>{\raggedright\arraybackslash}X}

\theoremstyle{plain}
\newtheorem{theorem}{Theorem}[section]
\newtheorem{proposition}[theorem]{Proposition}

\theoremstyle{definition}
\newtheorem{definition}[theorem]{Definition}
\newtheorem{assumption}[theorem]{Assumption}
\theoremstyle{remark}

\theoremstyle{problem}
\newtheorem{problem}[theorem]{Problem}

\title{ReAD: Reinforcement-Guided Capability Distillation for Large Language Models}

\author{%
\begin{tabular}{c}
\small
\textbf{Xueqi Cheng}$^{1}$ \quad
\textbf{Xugui Zhou}$^{2}$ \quad
\textbf{Tyler Derr}$^{3}$ \quad
\textbf{Yushun Dong}$^{1}$ \\
{\normalfont\footnotesize
$^{1}$Florida State University \quad
$^{2}$Louisiana State University \quad
$^{3}$Vanderbilt University} \\
{\normalfont\ttfamily\footnotesize
\{xc25,yushun.dong\}@fsu.edu;\ xuguizhou@lsu.edu;\ tyler.derr@vanderbilt.edu}
\end{tabular}%
}

\begin{document}
\maketitle

\begin{abstract}
Capability distillation applies knowledge distillation to selected model capabilities, aiming to compress a large language model (LLM) into a smaller one while preserving the abilities needed for a downstream task. However, most existing methods treat capabilities as independent training targets and overlook how improving one capability can reshape the student's broader capability profile, especially when multiple abilities jointly determine task success. We study capability distillation under a fixed token budget and identify two consistent patterns: distillation induces systematic, budget-dependent cross-capability transfer, and additional budget often brings limited task-relevant gains while sometimes degrading other useful abilities.
Building on these insights, we propose \textit{ReAD}, a \textit{Re}inforcement-guided c\textit{A}pability \textit{D}istillation framework that explicitly accounts for capability interdependence. ReAD first infers task-essential capabilities, then generates capability-targeted supervision on the fly, and finally uses an uncertainty-aware contextual bandit to adaptively allocate the distillation budget based on expected utility gains. Extensive experiments show that ReAD improves downstream utility under the same token budget while reducing harmful spillover and wasted distillation effort compared to strong baselines. Our code is publicly available at 
\url{https://github.com/LabRAI/ReAD}.
\end{abstract}

\section{Introduction}
\label{sec:intro}

Knowledge distillation (KD)~\cite{hinton2015distilling} for large language models (LLMs)~\cite{zhao2023survey} has become an essential research direction for enhancing the accessibility and efficiency of Machine-Learning-as-a-Service (MLaaS)~\cite{cai2024llmaas,xu2024survey_kd_llms,yang2024survey_kdllm}. Through knowledge distillation, a small LLM learns to imitate a large teacher LLM's outputs, enabling effective deployment under limited computational or financial resources~\cite{sanh2019distilbert,jiao2020tinybert,wang2020minilm,li2021dynamickd,liang2023less}. Recent studies show that capability distillation, a specialization of knowledge distillation that focuses supervision on a target capability (e.g., instruction-following, reasoning, mathematics, or coding), can substantially improve the student’s performance on downstream tasks that primarily depend on that capability, while achieving these gains at lower serving cost than the large teacher model~\cite{magister2023teaching,shridhar2022distilling,xu2023wizardlm,yue2024distilling}.

Despite rapid progress in capability distillation, most existing methods still treat capabilities as if they can be improved independently~\cite{taori2023stanford, chiang2023vicuna,zhang2024knowledgeable}, even though empirical evidence suggests that optimizing one capability often triggers broad, unintended shifts in the model’s overall capability profile~\cite {zhong2024revisiting, fang2025kddd_survey,cloud2025subliminal}. This mismatch is especially consequential in the budgeted setting: under a fixed token budget, the goal is to maximize downstream task utility, which depends not only on the targeted capability but also on other capabilities the task may implicitly require. If cross-capability interactions are ignored, allocating more tokens to a single target can become inefficient in two ways: the extra improvement on the target can shrink as the budget grows, and the same updates may reduce performance on other tasks that require non-target capabilities. In this case, additional tokens are effectively spent on updates that provide little task-relevant benefit and may even increase deployment risk in MLaaS, which we refer to as \emph{budget waste}. To systematically examine this gap, we take a first step toward measuring and understanding these interactions through a controlled empirical study. Concretely, we repeatedly distill the student under multiple token budgets, each time allocating the entire budget to one target capability drawn from a set of widely recognized core LLM capabilities: General Knowledge (General), Reasoning, Math, Code, Tool use, Long-Context Understanding (LCU), Steerability, and Multilinguality. After each run, we evaluate the resulting student on the full benchmark suite in Table~\ref{tab:capabilities-benchmarks} and record the score change for every capability. This yields a budget-dependent capability transfer matrix in which the diagonal entries capture on-target improvement and the off-diagonal entries capture how training toward one capability redistributes performance across the others. Across budgets, we observe two consistent patterns: \textit{(i) capability-specific distillation induces systematic, budget-dependent transfer to other capabilities rather than isolated improvements}, and \textit{(ii) increasing the budget for a single target produces smaller additional target gains while making the average harm to non-target capabilities more pronounced}, which together explain why naively scaling the budget for one capability can be inefficient.

Building on these observed insights, we propose \textit{ReAD}, a \textit{Re}inforcement-guided c\textit{A}pability \textit{D}istillation framework for large language models. Overall, ReAD combines on-the-fly capability-targeted data generation with token-level knowledge distillation, and uses a lightweight contextual bandit to adaptively allocate a fixed budget across interdependent capabilities. Specifically, ReAD first infers a task requirement vector that identifies which capabilities are essential for improving the downstream utility and treats degradations on these dimensions as harmful spillover. It then allocates distillation effort to maximize a proxy reward that favors capability gains aligned with the task requirements while penalizing spillover and budget consumption. In each interval, ReAD samples capability-labeled prompts according to the current allocation, queries the teacher to obtain supervision, updates the student with a standard distillation loss, and uses the resulting capability-profile change to update an uncertainty-aware UCB allocation rule. Extensive experiments demonstrate that ReAD strengthens task-relevant capabilities and reduces wasted budget on low-utility capability updates under budget constraints.

To sum up, this paper makes the following key contributions: 

\begin{itemize}[leftmargin=*, itemsep=1pt, topsep=1pt]
    \item \textbf{Empirical study of cross-capability interactions.}
    We show that distilling toward a single capability consistently changes other capabilities, revealing systematic cross-capability interactions that are often overlooked.
    
    \item \textbf{Budgeted capability distillation formulation.}
    We formulate capability distillation as allocating a fixed budget across multiple interacting capabilities, providing objective for improving utility while controlling side effects.
    
    \item \textbf{ReAD: reinforcement-guided capability distillation.}
    ReAD infers task-relevant capabilities, generates capability-targeted distillation data with controllable style and difficulty, and uses an uncertainty-aware contextual bandit to allocate the budget across capabilities.
    
    \item \textbf{Theory analysis of the capability interdependence.}
    We explain cross-capability interactions and diminishing returns in capability distillation.
\end{itemize}
\section{Preliminary}
\label{sec:prelim}

\subsection{Notation}

In this paper, we study budgeted capability distillation from a large teacher model \(T\) to a
smaller student model \(S\), with total training-token budget \(B\). Let
\(\mathcal{C}=\{c_1,\ldots,c_{|\mathcal{C}|}\}\) denote a set of measurable
capabilities, and let \(s_k(M)\) be model \(M\)'s benchmark score on capability
\(c_k\). A distillation strategy \(\pi\) specifies how data are generated, how
the teacher is queried, and how tokens are allocated across capabilities,
assigning budgets \(\{b_k\}\) with \(\sum_k b_k\le B\), and producing a
distilled student \(S_\pi\).

\subsection{Intuition and Formalization}

Existing work on LLM capability distillation often assumes that targeted
capabilities remain independent under compression. However, distilling for one
capability can inadvertently shift others, leading to inefficient or even
counterproductive use of a limited budget~\cite{zhong2024revisiting,
fang2025kddd_survey,cloud2025subliminal}. To study these interactions, we
evaluate eight core LLM capabilities using commonly adopted benchmarks, with
the full benchmark-to-metric mapping provided in Appendix~\ref{sec:bench_datasets},
Table~\ref{tab:capabilities-benchmarks}. These capabilities span major dimensions
of LLM behavior and motivate our unified \emph{capability distillation}
formulation.

\begin{definition}[Capability distillation]\label{def:cap_distill}
Given a teacher model \( T \), an initial student model \( S_0 \), and
capability-specific data distributions \( \{ \mathcal{D}_c \}_{c\in\mathcal{C}} \),
capability distillation allocates a limited budget \( B \) across capabilities
to improve task-relevant performance under shared representations by optimizing
a weighted mixture of distillation objectives:
\begin{equation}
\begin{aligned}
\min_{S,\,\mathbf{w}} \;&
\mathbb{E}_{x\sim \mathcal{D}_{\mathbf{w}}}
\!\left[\ell\!\left(S(x),T(x)\right)\right] \quad\
\text{s.t.} \quad\
\mathbf{w}\in\Delta_{|\mathcal{C}|},\; \mathrm{cost}(S;S_0)\le B .
\end{aligned}
\end{equation}
where \( \mathcal{D}_{\mathbf{w}} := \sum_{c\in\mathcal{C}} w_c \mathcal{D}_c \), and $\mathrm{cost}(S;S_0)$ measures the distillation token budget consumed to obtain $S$ from $S_0$.
Here \( \mathbf{w}=(w_c)_{c\in\mathcal{C}} \) lies in the probability simplex
\[
\Delta_{|\mathcal{C}|}=\Big\{\mathbf{w}\in\mathbb{R}^{|\mathcal{C}|}\ \big|\ w_c\ge 0,\ \sum_{c\in\mathcal{C}} w_c=1\Big\},
\]
which specifies the allocation of training effort across capabilities, and
\( \ell \) denotes a standard distillation loss, such as token-level
cross-entropy or logit-matching between the student and teacher outputs.
\end{definition}

We therefore conduct an empirical study to identify and quantify how allocating a limited distillation budget to specific capabilities induces systematic performance changes in other capabilities, thereby exposing the intrinsic interdependence structure among capabilities in LLMs.
\begin{definition}[Capability interdependence]\label{def:cap_interdependence}
Let \( S_0 \) be an initial student model, and let \( S(\mathbf{w},B) \) denote
the student obtained after capability distillation under allocation
\( \mathbf{w}\in\Delta_{|\mathcal{C}|} \) and budget \( B \). For each capability
\( c_i \), define the one-hot allocation vector
\( \mathbf{w}^{(i)}\in\Delta_{|\mathcal{C}|} \) by \( w^{(i)}_i=1 \) and
\( w^{(i)}_j=0 \) for all \( j\neq i \). The \emph{capability interdependence}
under capability distillation is characterized by the capability transfer matrix
\[
\mathbf{T}_{ij}(B)
:= s_j\!\left(S(\mathbf{w}^{(i)},B)\right) - s_j(S_0),
\]
where $\mathbf{T}(B)\in\mathbb{R}^{|\mathcal{C}|\times|\mathcal{C}|}$, \( \mathbf{T}_{ij}(B) \) quantifies the effect of distilling
toward capability \( c_i \) on the performance of capability \( c_j \).
Non-zero off-diagonal entries indicate interdependence between capabilities
under shared representations.
\end{definition}






\vspace{-1em}

\subsection{Exploratory Study and Problem Formulation}
~\label{sec:emp_res}

\vspace{-2em}


In this section, we analyze empirical capability interdependence through controlled distillation
experiments. Across all experiments, we distill Llama-3.3-70B-Instruct into
Llama-3.1-8B-Instruct under fixed token budgets, using the target capabilities
in Appendix Table~\ref{tab:capabilities-benchmarks} and several representative
capability-distillation strategies. For each strategy \( \pi \) and budget
\( B \), we evaluate the student on all benchmarks to form a capability transfer
matrix, capturing both target-capability gains and non-target changes. We
normalize all benchmark scores to \( [0,100] \); details are provided in
Appendix~\ref{sec:app_em}.

\begin{wrapfigure}[17]{r}{0.55\textwidth}
  \vspace{-1.2em}
  \centering

  \includegraphics[width=0.9\linewidth]{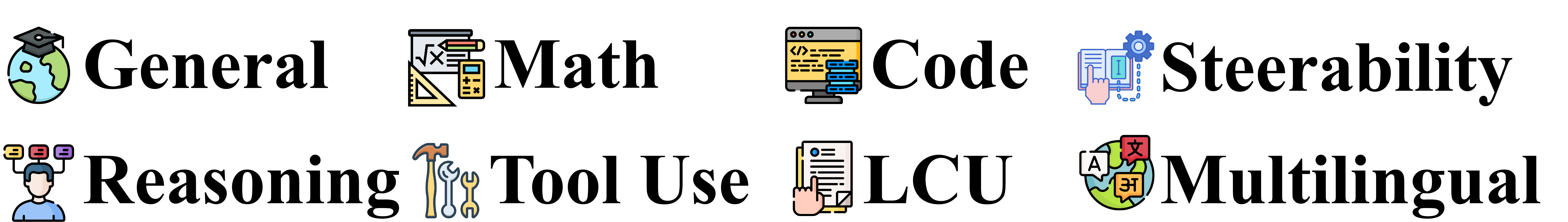}
  \vspace{0.35em}

  \begin{subfigure}[t]{0.445\linewidth}
    \centering
    \includegraphics[width=\linewidth]{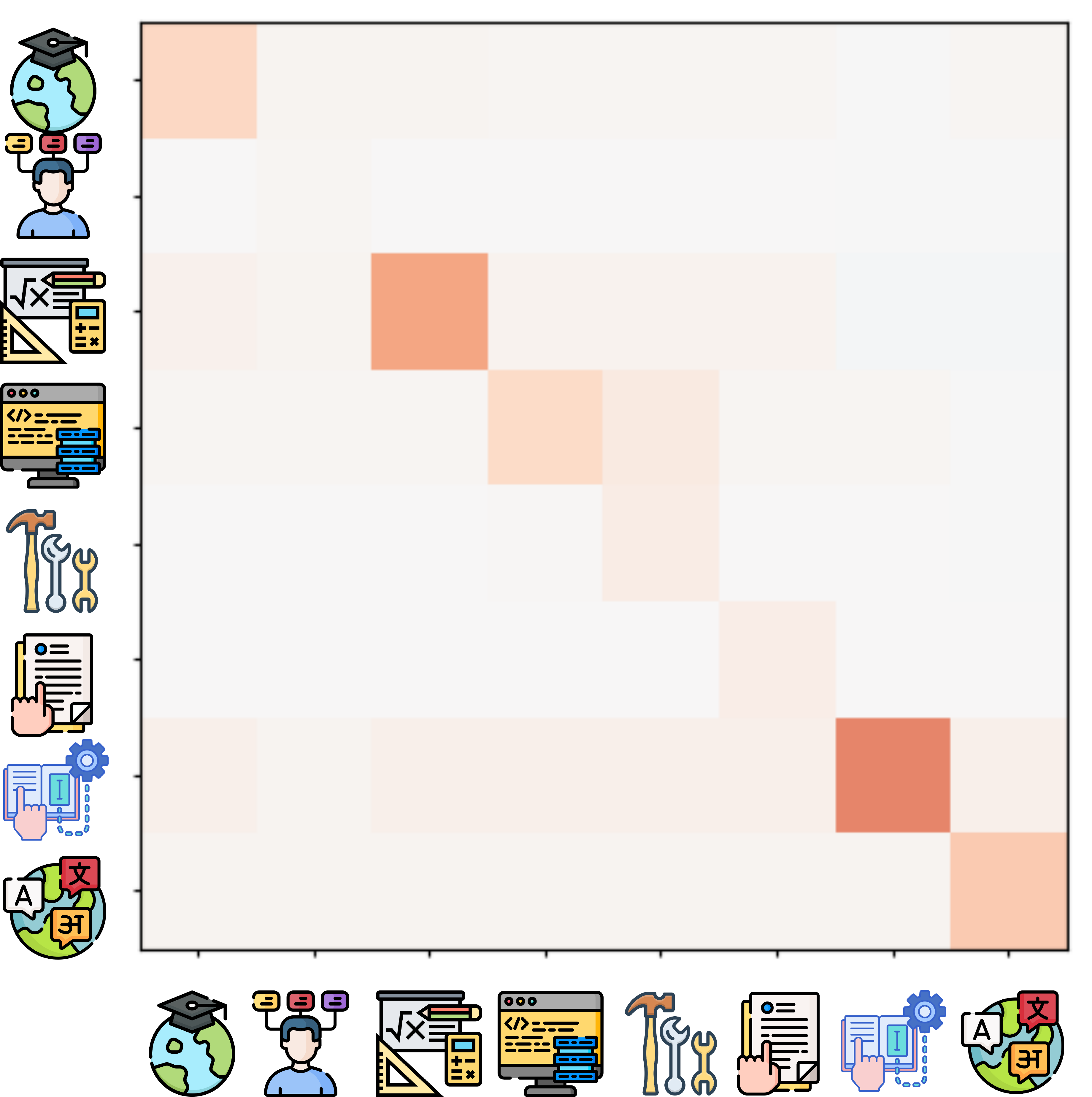}
    \caption{Small token budget}
    \label{fig:heatmap_small}
  \end{subfigure}
  \hfill
  \begin{subfigure}[t]{0.525\linewidth}
    \centering
    \includegraphics[width=\linewidth]{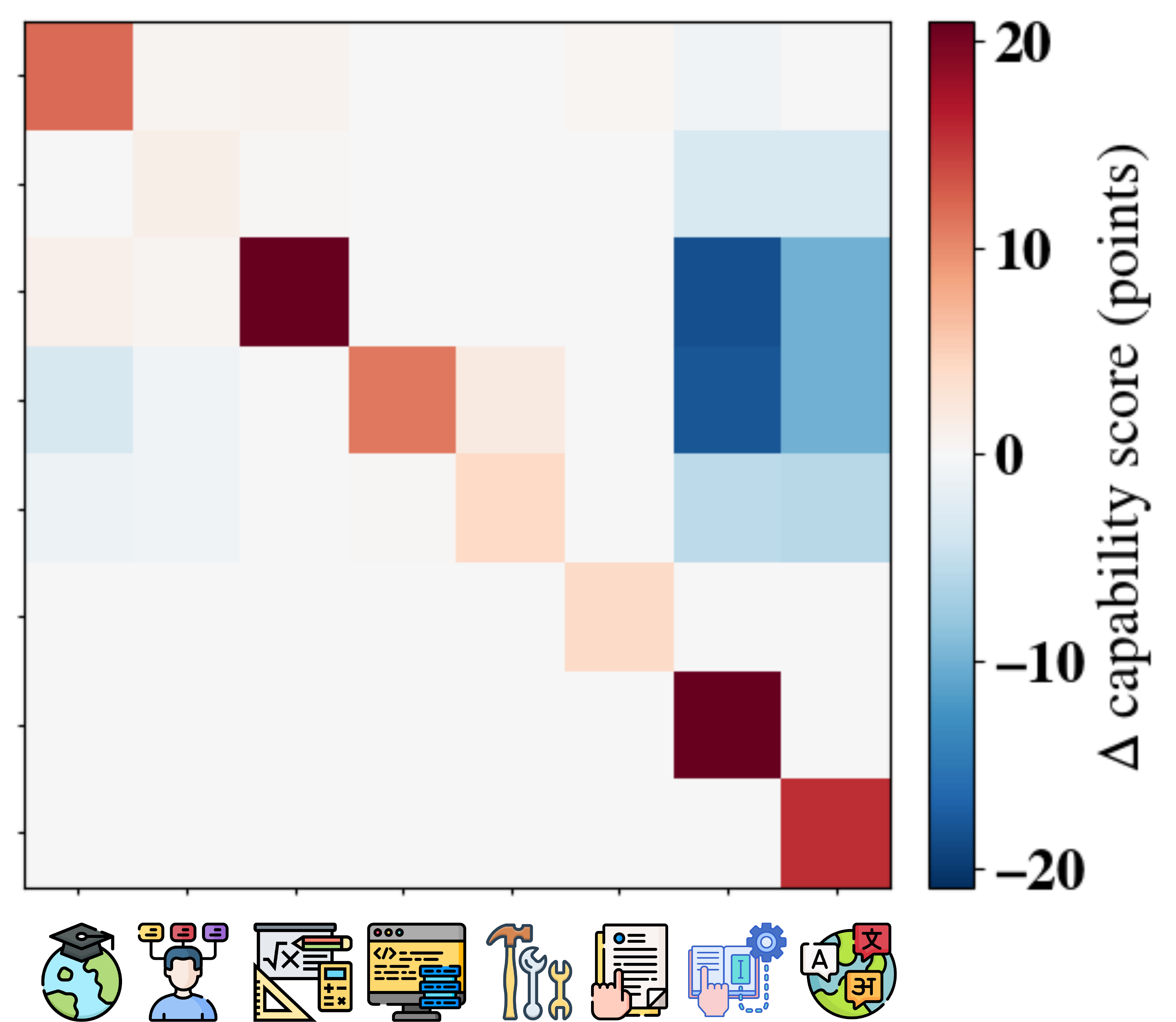}
    \caption{Large token budget}
    \label{fig:heatmap_large}
  \end{subfigure}

  \caption{
  Cross-capability transfer under single-capability distillation.
  Larger budgets sharpen target-capability gains but expose stronger negative transfer.
  }
  \label{fig:capability_transfer_heatmaps}
\end{wrapfigure}

\textbf{Observation 1: Distilling a specific capability redistributes performance across other capabilities in a budget-dependent manner.} Figure~\ref{fig:capability_transfer_heatmaps} reports the visualization of the capability transfer matrices under different token budgets. Here, the off-diagonal mass is consistently non-zero, meaning optimizing a single capability alters other capabilities rather than leaving them unchanged. Besides, the structure varies with budget: at small budgets, changes are generally weak and diffuse, whereas at large budgets the diagonal strengthens, but negative off-diagonals become more visible for specific non-targets. Together, these results indicate that capability distillation induces a structured redistribution of performance.

\begin{wrapfigure}[13]{r}{0.6\textwidth}
  \centering

  \begin{subfigure}[t]{0.515\linewidth}
    \centering
    \includegraphics[width=\linewidth]{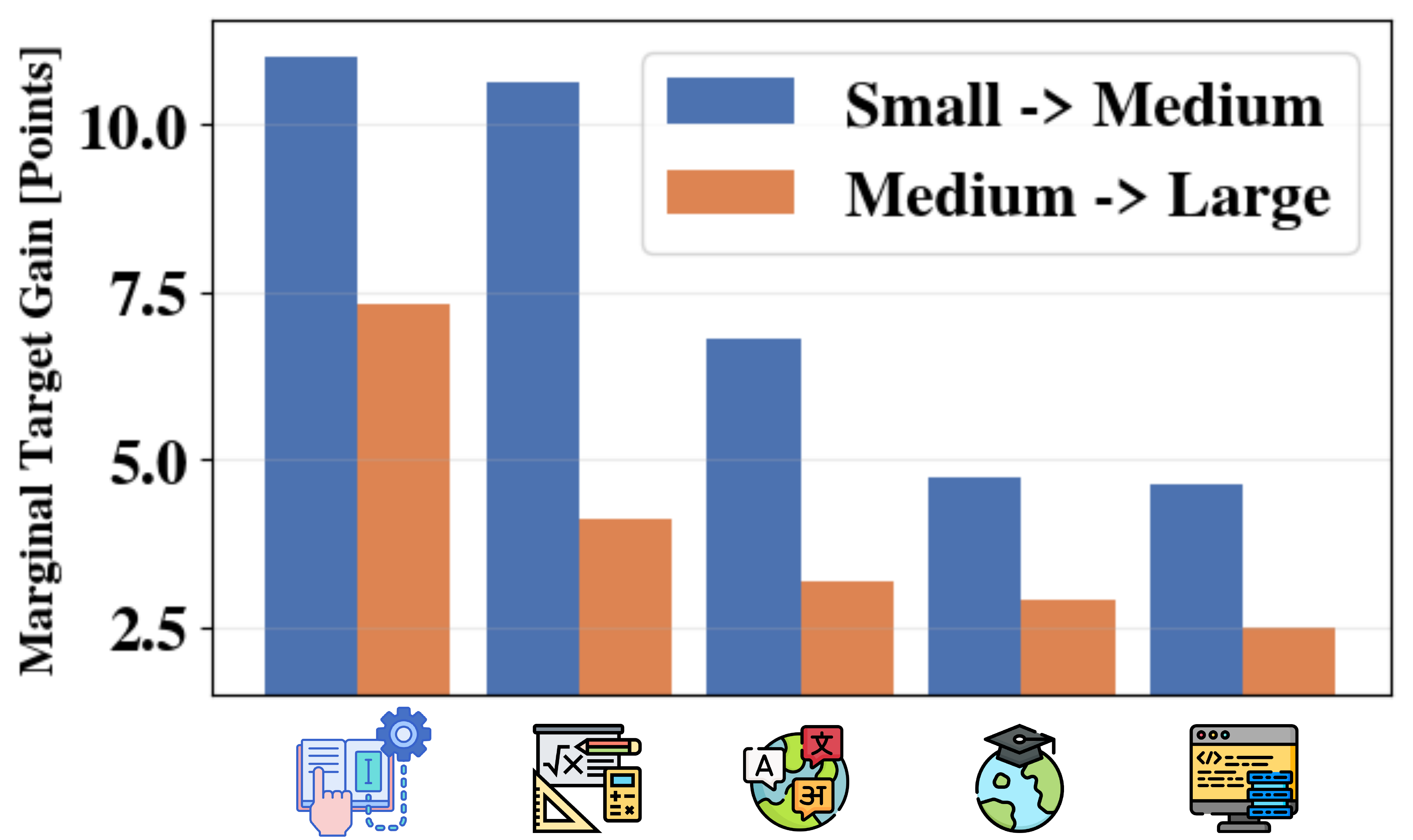}
    \caption{Diminishing returns}
    \label{fig:diminishing_returns}
  \end{subfigure}
  \hfill
  \begin{subfigure}[t]{0.455\linewidth}
    \centering
    \includegraphics[width=\linewidth]{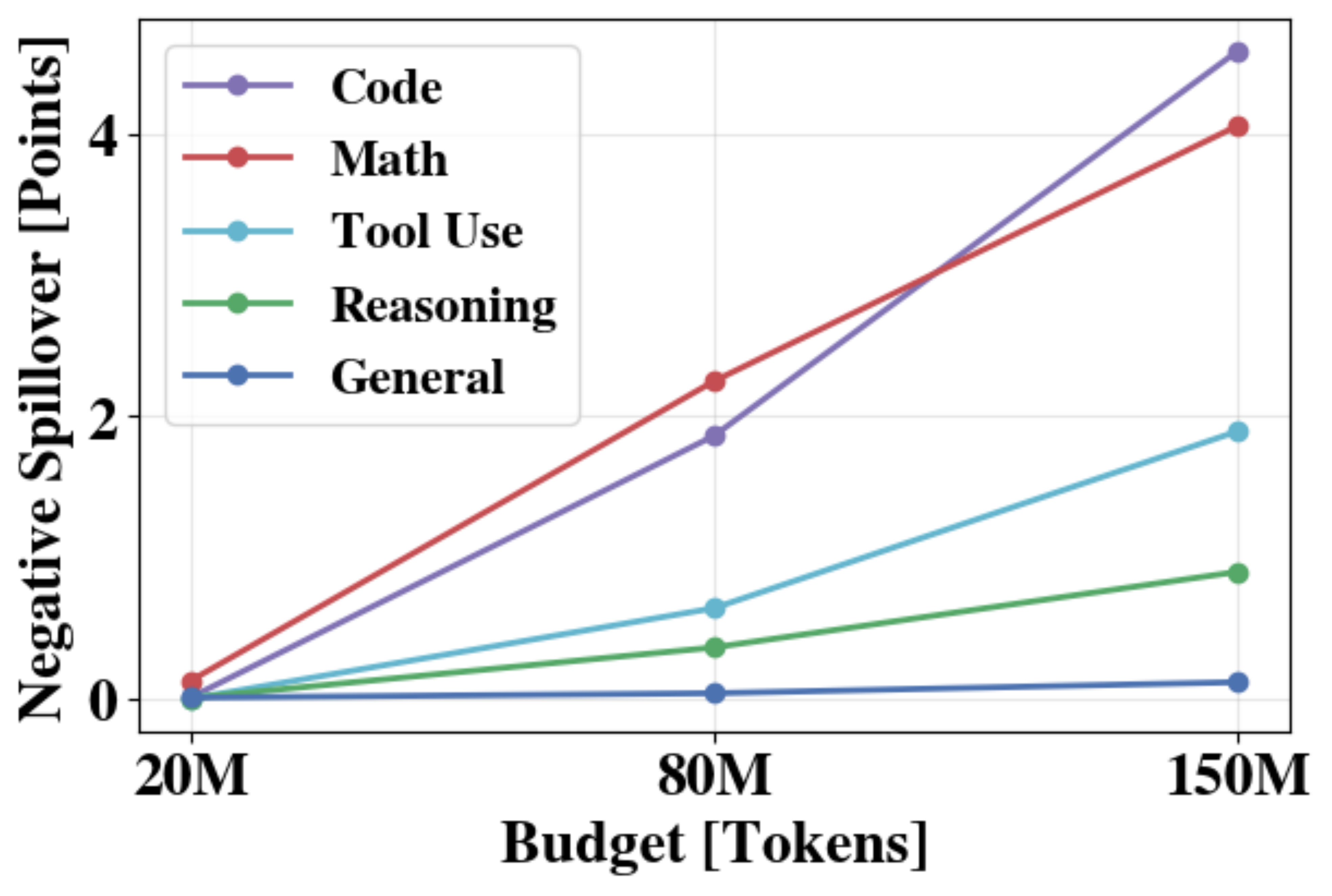}
    \caption{Rising spillover}
    \label{fig:rising_spillover}
  \end{subfigure}

  \caption{
  Budget waste in capability distillation. Extra tokens yield smaller target gains
  while increasing harmful spillover to non-target capabilities.
  }
  \label{fig:budget_waste_mechanism}
\end{wrapfigure}

\textbf{Observation 2: Capability distillation exhibits substantial budget waste
due to diminishing returns and negative spillover.}
Under a fixed distillation budget, additional tokens for a target capability
should ideally provide meaningful target improvement without degrading other
capabilities needed downstream. We test this using two diagnostics derived from
the capability transfer matrix \(T(B)\). First, \emph{diminishing returns}
measures whether marginal target gains shrink as the budget increases, comparing
20M\(\rightarrow\)80M against 80M\(\rightarrow\)150M for the same target.
Second, \emph{negative spillover} measures the drop on non-target capabilities
as more budget is assigned to target.

Figure~\ref{fig:diminishing_returns} shows that, for the strongest-improving
targets, the 20M\(\rightarrow\)80M gain is consistently larger than the
80M\(\rightarrow\)150M gain, indicating target-capability saturation.
Figure~\ref{fig:rising_spillover} further shows that collateral harm to
non-target capabilities grows with budget. Together, these trends reveal
budget waste: beyond a capability-specific knee point, extra tokens yield
limited target benefit while amplifying cross-capability degradation, motivating
allocation strategies that explicitly trade off target gains against spillover.

Overall, these observations suggest that effective capability distillation under a fixed budgets requires jointly reasoning about task-relevant capabilities, cross-capability interactions, and budget allocation. We formalize this as a budget allocation problem across interdependent capabilities:

\begin{problem}[Budget allocation for interdependent capability distillation]\label{prob:budget_alloc}
Given a teacher model \( T \), an initial student model \( S_0 \), a downstream
task \( \tau \), and a total distillation budget \( B \), consider a sequential
distillation process over steps \( t=1,\ldots,T \) with cumulative cost not
exceeding \( B \). At each step \( t \), a capability allocation vector
\( \mathbf{w}_t\in\Delta_{|\mathcal{C}|} \) is selected and applied to update
the student, yielding \( S_t \). Let \( U_\tau(\cdot) \) is the task-dependent utility function, our target is to determine an allocation policy
\( \{\mathbf{w}_t\}_{t=1}^T \) that maximizes the performance of
the final student model:
\begin{equation}
\max_{\{\mathbf{w}_t\}_{t=1}^T}\ U_\tau\!\left(S_T\right), \
\text{s.t.}\;
\sum_{t=1}^T \mathrm{cost}\!\left(S_{t-1}\!\rightarrow\!S_t\right)\le B
\end{equation}
\end{problem}




\section{Methodology}
\label{sec:method}

ReAD treats Problem~\ref{prob:budget_alloc} as state-dependent budget allocation
over capabilities. At each step, it estimates task requirements, generates
capability-targeted teacher supervision, and uses an uncertainty-aware contextual
bandit to choose the next token allocation. As the student profile changes, ReAD
can shift budget away from saturated or high-spillover capabilities.

\subsection{Identifying Task-Essential Capabilities}
\label{sec:method:essential}

\paragraph{Task requirement vector.}
For each downstream task $\tau$, ReAD constructs a task card
$\mathcal{D}^{\mathrm{spec}}_\tau$ from non-test data, containing a short task
description, input/output format, evaluation target, and a few representative
exemplars, with the test split reserved only for final reporting. From this
task card, ReAD estimates a requirement vector
$\mathbf{r}_\tau\in\Delta_{|\mathcal{C}|}$, where $r_{\tau,c}$ indicates how
useful capability $c$ is for improving the task utility $U_\tau(\cdot)$. We use
$\mathbf{r}_\tau$ as an actionable allocation prior: high-mass capabilities are
prioritized for improvement, and drops on these capabilities are treated as
harmful spillover. Since token allocations cannot be negative,
$\mathbf{r}_\tau$ is constrained to the simplex; negative sensitivities are
captured instead by signed capability changes and the spillover penalty in
Section~\ref{sec:method:bandit}.

\paragraph{Requirement identifier.}
We learn a lightweight identifier $g_\phi$ that maps the task card to the requirement vector $\mathbf{r}_\tau = g_\phi(\mathcal{D}^{\mathrm{spec}}_\tau) \in \Delta_{|\mathcal{C}|}$.
The identifier is a small Transformer encoder over the task card, followed by a two-layer MLP and a softmax head. Since $\mathbf{r}_\tau$ is not directly observed, we supervise it through low-budget interventions in capability space.

\paragraph{Local supervision signal.}
Let $\mathbf{s}(S)\in\mathbb{R}^{|\mathcal{C}|}$ denote the measured capability profile of a student $S$, and write $U_\tau(S)=F_\tau(\mathbf{s}(S))$. For a small intervention that changes the profile from $\mathbf{s}$ to $\mathbf{s}+\Delta\mathbf{s}$,
\begin{equation}
F_\tau(\mathbf{s}+\Delta\mathbf{s})
=
F_\tau(\mathbf{s})+
\nabla_{\mathbf{s}}F_\tau(\mathbf{s})^\top\Delta\mathbf{s}
+o(\|\Delta\mathbf{s}\|),
\label{eq:taylor}
\end{equation}
so the utility change is locally approximated by
\begin{equation}
\Delta U_\tau
=
U_\tau(S')-U_\tau(S)
\approx
\nabla_{\mathbf{s}}F_\tau(\mathbf{s})^\top\Delta\mathbf{s}.
\label{eq:local_linear}
\end{equation}
Thus, $\mathbf{r}_\tau$ is trained as a simplex-constrained surrogate for the beneficial part of the local utility sensitivity, e.g., $\mathbf{r}_\tau\approx\Pi_\Delta([\nabla_{\mathbf{s}}F_\tau(\mathbf{s})]_+)$. This does not assume that the true sensitivity has no negative coordinates; it only converts the sensitivity into a nonnegative budget-allocation prior.

\paragraph{Offline pretraining.}
We construct an intervention set
$\mathcal{Z}=\{(\mathcal{D}^{\mathrm{spec}}_\tau,\Delta\mathbf{s}^{(m)}_\tau,\Delta U^{(m)}_\tau)\}_{\tau,m}$ from auxiliary tasks and non-test splits. For each task $\tau$, we sample sparse allocation vectors $\mathbf{w}^{(m)}\in\Delta_{|\mathcal{C}|}$ by choosing a support set $\mathcal{I}^{(m)}\subseteq\mathcal{C}$, sampling nonzero weights from a Dirichlet distribution, and renormalizing them. A short probe distillation under a small budget $b_{\mathrm{probe}}$ produces an intervened student $S^{(m)}$, from which we measure
\begin{equation}
\Delta U^{(m)}_\tau := U_\tau(S^{(m)})-U_\tau(S_0),
\qquad
\Delta\mathbf{s}^{(m)}_\tau := \mathbf{s}(S^{(m)})-\mathbf{s}(S_0).
\label{eq:probe_deltas}
\end{equation}
We pretrain $g_\phi$ by predicting the observed utility change from the induced capability change:
\begin{equation}
\min_\phi
\sum_{(\tau,m)}
\Big(\Delta U^{(m)}_\tau
- g_\phi(\mathcal{D}^{\mathrm{spec}}_\tau)^\top\Delta\mathbf{s}^{(m)}_\tau\Big)^2
+\lambda_{\mathrm{ent}}\,\mathcal{H}\!\left(g_\phi(\mathcal{D}^{\mathrm{spec}}_\tau)\right).
\label{eq:req_pretrain}
\end{equation}
The positive entropy penalty discourages near-uniform requirement vectors and makes the allocation prior more discriminative. At deployment time, ReAD computes $\mathbf{r}_\tau$ once and reuses it throughout the distillation run. Since the identifier is trained once and reused across tasks, it is not counted as part of the per-run student distillation budget.

\subsection{On-the-Fly Capability-Targeted Data Generation}
\label{sec:method-generator}

\paragraph{Capability-conditioned templates.}
Given $\mathbf{r}_\tau$, ReAD generates training examples on the fly from a capability-labeled template library. For each capability $c\in\mathcal{C}$, a template set $\mathcal{P}_c$ specifies an instruction scaffold, typed slots, and output-format constraints. Slot values are filled using deterministic seeds. The generator controls the form of teacher supervision

\paragraph{Difficulty scoring.}
For prompt $x$ in capability $c$, we compute a deterministic difficulty score
\begin{equation}
d_c(x)=\frac{1}{|\mathcal{F}_c|}\sum_{f\in\mathcal{F}_c}
\frac{f(x)-\min_{x'\in\mathcal{Q}_c} f(x')}
{\max_{x'\in\mathcal{Q}_c} f(x')-\min_{x'\in\mathcal{Q}_c} f(x')+\epsilon},
\label{eq:difficulty_score}
\end{equation}
where $\mathcal{F}_c$ is the set of capability-specific control factors and $\mathcal{Q}_c$ is a calibration prompt pool for capability $c$. For example, reasoning templates use factors such as the number of constraints, reasoning hops, and symbolic complexity; code templates use the number of required functions and tests; steerability templates use the number of rules and schema fields. We split each capability's calibration pool into easy, medium, and hard buckets by tertiles of $d_c(x)$.

\paragraph{Curriculum and sampling.}
At decision step $t$, ReAD samples a capability label according to the selected allocation $\mathbf{w}_t$, instantiates a prompt from the corresponding template set, queries the teacher $T$ for a completion $y$, and immediately distills on $(x,y)$. The curriculum only adjusts the within-capability difficulty mix: early steps emphasize easy and medium examples, while later steps increase the probability of hard examples. It does not choose the capability allocation. We also track recent template frequencies and cap each template's share so that no single scaffold dominates.

\subsection{Contextual Bandit for Capability Allocation}
\label{sec:method:bandit}

\paragraph{Decision loop.}
ReAD runs for $T_{\mathrm{step}}=20$ decision steps. Thus, the bandit update, data resampling, and proxy refresh occur every $1.0$M tokens in the $20$M setting and every $7.5$M tokens in the $150$M setting. At step $t$, the context summarizes the task requirement, current student profile, remaining budget, and recent allocation history:
\begin{equation}
\mathbf{x}_t=[\mathbf{r}_\tau;\mathbf{s}^{\mathrm{probe}}(S_t);b_t;\boldsymbol{\rho}_t],
\label{eq:bandit_context}
\end{equation}
where $\boldsymbol{\rho}_t$ stores a short record of recent allocations and observed gains. The profile $\mathbf{s}^{\mathrm{probe}}(S_t)$ is computed with a small fixed monitoring suite, while full held-out benchmark evaluation is used only for final reporting and coarse calibration.

\paragraph{Candidate allocations and distillation.}
At each step, ReAD selects an allocation $\mathbf{w}_t$ from a finite candidate set $\mathcal{A}(\tau)$ consisting of local perturbations around the previous allocation and sparse actions concentrated on the top-$k$ capabilities under $\mathbf{r}_\tau$. Weights are chosen from a fixed grid and renormalized to sum to one. Given $\mathbf{w}_t$, ReAD samples capability-targeted data and updates the student by minimizing $\mathcal{L}_{\mathrm{distill}}(\theta_t)
=-\mathbb{E}_{(x,y)}\sum_{j=1}^{|y|}
\log p_{\theta_t}(y_j\mid x,y_{<j})$ to produce $S_{t+1}$.

\paragraph{Proxy reward.}
After the update, ReAD measures the probe-profile change
$\Delta\mathbf{s}^{\mathrm{probe}}_t=\mathbf{s}^{\mathrm{probe}}(S_{t+1})-\mathbf{s}^{\mathrm{probe}}(S_t)$ and forms
\begin{equation}
\widehat{R}_t
=\mathbf{r}_\tau^\top\Delta\mathbf{s}^{\mathrm{probe}}_t
-\beta\,\mathrm{Spill}_t
-\lambda\,\mathrm{cost}_t, \quad\ \mathrm{Spill}_t=
\sum_{c\in\mathcal{C}_{\mathrm{ess}}}
 r_{\tau,c}\,[-\Delta s^{\mathrm{probe}}_{t,c}]_+
\label{eq:reward}
\end{equation}
Here $\mathcal{C}_{\mathrm{ess}}$ contains the top-$k$ capabilities under $\mathbf{r}_\tau$, and $\mathrm{cost}_t$ is the token budget consumed at the step. The first term rewards task-aligned capability gains, the second penalizes regressions on task-essential capabilities, and the third enforces budget awareness.

\paragraph{Reward model and UCB allocation.}
We append $(\mathbf{x}_t,\mathbf{w}_t,\widehat{R}_t)$ to a history buffer and train an ensemble of $J$ two-layer MLP reward regressors $\{h_{\eta_j}\}_{j=1}^J$. Each regressor maps $(\mathbf{x},\mathbf{w})$ to a scalar next-step reward and is trained on a bootstrap resample of the 10\% profiling split plus logged transitions from earlier steps, yielding about $3.2$k--$4.9$k examples per task in our experiments. For a candidate allocation $\mathbf{w}$, the ensemble mean and uncertainty are
\begin{equation}
\mu(\mathbf{x}_t,\mathbf{w})=\frac{1}{J}\sum_{j=1}^J h_{\eta_j}(\mathbf{x}_t,\mathbf{w}),
\quad
\sigma(\mathbf{x}_t,\mathbf{w})=
\sqrt{\frac{1}{J-1}\sum_{j=1}^J\big(h_{\eta_j}(\mathbf{x}_t,\mathbf{w})-\mu(\mathbf{x}_t,\mathbf{w})\big)^2}.
\label{eq:ensemble_mu_sigma}
\end{equation}
ReAD then selects the next allocation using an upper-confidence rule,
\begin{equation}
\mathbf{w}_{t+1}
=\arg\max_{\mathbf{w}\in\mathcal{A}(\tau)}
\mu(\mathbf{x}_{t+1},\mathbf{w})+\kappa\,\sigma(\mathbf{x}_{t+1},\mathbf{w}),
\label{eq:ucb}
\end{equation}
with $b_{t+1}=b_t-\mathrm{cost}_t$. At coarse development checkpoints, we fit an affine calibration from cumulative proxy reward to observed utility change and reduce exploration if the development utility stagnates. Unlike static mixture-regression approaches such as RegMix~\cite{liu2024regmix}, ReAD repeatedly re-estimates the best next allocation from the current student state and remaining budget.

\section{Theoretical Analysis}
\label{sec:theory}

This section provides a local theoretical account of ReAD's allocation problem.
We characterize cross-capability transfer under shared representations, show
when diminishing returns make continued allocation wasteful, and motivate
uncertainty-aware allocation under noisy reward estimates.

\subsection{Local Cross-Capability Transfer}
\label{subsec:theory_transfer_short}

Let $\theta_t\in\mathbb{R}^p$ denote the parameters of the student $S_t$, and let $\Delta\theta_t=\theta_t-\theta_{t-1}$. We analyze one decision step, which is the regime in which ReAD re-estimates its allocation.

\begin{assumption}[Mixture-structured update]
At interval $t$, given an allocation vector $\mathbf{w}_t\in\Delta_{|\mathcal{C}|}$, the expected update admits a local mixture approximation $\mathbb{E}[\Delta\theta_t\mid \mathbf{w}_t,\theta_{t-1}]
\approx
\sum_{c\in\mathcal{C}} w_{t,c}\,d_c(\theta_{t-1})$, where $d_c(\theta)$ is the capability-conditioned local update direction.
\end{assumption}

\begin{assumption}[Local capability readout]
The measured capability profile $\mathbf{s}(S)$ is represented locally as a differentiable function $\tilde{\mathbf{s}}(\theta)$. Near $\theta_{t-1}$,
\begin{equation}
\Delta s_{t,c}
=
[\mathbf{s}(S_t)]_c-[\mathbf{s}(S_{t-1})]_c
\approx
\nabla_\theta[\tilde{\mathbf{s}}(\theta_{t-1})]_c^\top\Delta\theta_t.
\label{eq:local_linear_short}
\end{equation}
\end{assumption}

\begin{proposition}[Local transfer decomposition]
Under the two local approximations above, the expected change in capability $c$ satisfies
\begin{equation}
\Delta s_{t,c}
\approx
\sum_{c'\in\mathcal{C}} w_{t,c'}\,\Gamma_{c,c'}(\theta_{t-1}),
\qquad
\Gamma_{c,c'}(\theta)
:=
\nabla_\theta[\tilde{\mathbf{s}}(\theta)]_c^\top d_{c'}(\theta).
\label{eq:transfer_decomp_short}
\end{equation}
\end{proposition}

\noindent\textit{Interpretation.}
The matrix $\Gamma$ is the local analogue of the empirical capability-transfer matrix. Diagonal terms capture intended gains, while off-diagonal terms capture transfer or interference. Off-diagonal terms are nonzero whenever the update direction for capability $c'$ is not orthogonal to the readout direction of capability $c$, which is expected under shared representations. This is the decision-step explanation for why allocating budget to one capability can alter others.

\subsection{Budget Waste from Diminishing Returns}
\label{subsec:theory_diminish_short}

The previous subsection explains spillover. We now isolate diminishing returns with a simpler cumulative-budget proxy. Let $b_c\ge0$ be the total number of distillation tokens assigned to capability $c$, with $\sum_c b_c\le B$, and let $\mathbf{r}_\tau\in\Delta_{|\mathcal{C}|}$ be the task requirement vector.

\begin{assumption}[Concave task-aligned gains]
\label{assump:concave_gains_short}
For task $\tau$ and capability $c$, there is a nondecreasing concave gain function $G_{\tau,c}:\mathbb{R}_{\ge0}\to\mathbb{R}_{\ge0}$ such that $\mathbb{E}[U_\tau(S)-U_\tau(S_0)]
\approx
\sum_{c\in\mathcal{C}} r_{\tau,c}G_{\tau,c}(b_c)$ where $\sum_{c\in\mathcal{C}} b_c\le B$.
\end{assumption}

Consider a restricted distillation strategy that allocates budget only to a subset $\mathcal{K}\subseteq\mathcal{C}$:
\begin{equation}
\max_{\{b_c\ge0\}_{c\in\mathcal{K}}}
\sum_{c\in\mathcal{K}} r_{\tau,c}G_{\tau,c}(b_c)
\quad\mathrm{s.t.}\quad
\sum_{c\in\mathcal{K}} b_c\le B.
\label{eq:alloc_proxy_subset}
\end{equation}

\begin{proposition}[Weighted marginal gain criterion]
\label{prop:kkt_subset}
Any optimum $\{b_c^\star\}_{c\in\mathcal{K}}$ of Eq.~\eqref{eq:alloc_proxy_subset} equalizes weighted marginal gains among active capabilities: there exists $\nu^\star\ge0$ such that
\begin{equation}
r_{\tau,c}G'_{\tau,c}(b_c^\star)\le\nu^\star
\quad\forall c\in\mathcal{K},
\qquad
r_{\tau,c}G'_{\tau,c}(b_c^\star)=\nu^\star
\quad\text{if } b_c^\star>0.
\label{eq:kkt_subset}
\end{equation}
Moreover, in the single-capability case targeting $\bar c$, shifting an infinitesimal budget from $\bar c$ to another task-relevant capability $c'$ improves the proxy objective whenever $r_{\tau,c'}G'_{\tau,c'}(0)
>
r_{\tau,\bar c}G'_{\tau,\bar c}(b_{\bar c})$.
\end{proposition}

\noindent\textit{Implication for ReAD.}
This analysis formalizes budget waste: once the target capability saturates, additional tokens can have lower task-aligned marginal value than tokens assigned elsewhere. The actual ReAD reward in Eq.~\eqref{eq:reward} is richer than this proxy because it also penalizes observed spillover. Thus, the theory motivates adaptive allocation, while the allocator directly optimizes a step-level reward that accounts for both diminishing returns and cross-capability degradation.

\subsection{Why Uncertainty-Aware Allocation Helps}
\label{subsec:theory_ucb_short}

At step $t$, ReAD chooses from a finite action set $\mathcal{A}(\tau)$ using the ensemble mean $\mu(\mathbf{x}_t,\mathbf{w})$ and uncertainty $\sigma(\mathbf{x}_t,\mathbf{w})$ in Eq.~\eqref{eq:ucb}. Let $\overline{R}(\mathbf{x},\mathbf{w})
:=
\mathbb{E}[\widehat{R}_t\mid \mathbf{x}_t=\mathbf{x},\mathbf{w}_t=\mathbf{w}]$ be the conditional expected proxy reward.

\begin{assumption}[Calibrated ensemble uncertainty]
\label{ass:calibrated_uncertainty}
For all encountered $(\mathbf{x}_t,\mathbf{w})$, with high probability, $|\mu(\mathbf{x}_t,\mathbf{w})-\overline{R}(\mathbf{x}_t,\mathbf{w})|
\le
\kappa\sigma(\mathbf{x}_t,\mathbf{w})$.
\end{assumption}

\begin{proposition}[Step-level regret bound]
Let $\mathbf{w}_t^\star\in\arg\max_{\mathbf{w}\in\mathcal{A}(\tau)}\overline{R}(\mathbf{x}_t,\mathbf{w})$. Under Assumption~\ref{ass:calibrated_uncertainty}, the action selected by Eq.~\eqref{eq:ucb} satisfies $\overline{R}(\mathbf{x}_t,\mathbf{w}_t^\star)
-
\overline{R}(\mathbf{x}_t,\mathbf{w}_t)
\le
2\kappa\sigma(\mathbf{x}_t,\mathbf{w}_t)$.
\end{proposition}

\noindent\textit{Proof sketch.}
By calibration, the true reward of any action is at most its upper confidence score, and the selected action maximizes this score. Applying the same calibration inequality to the selected action yields the bound. This result does not imply a global optimum for nonconvex fine-tuning; it justifies the local ranking rule used by ReAD at each budget step.

\section{Experimental Setting}
\label{subsec:exp_setup}

\textbf{Models and budgets.}
We mainly distill Llama-3.3-70B-Instruct into Llama-3.1-8B-Instruct, and report
Qwen2.5-72B-Instruct to Qwen2.5-14B-Instruct results in
Appendix~\ref{sec:app_experiment}. All methods use the same student checkpoint,
training recipe, and $20$M/$150$M distillation-token budgets.

\textbf{Evaluation protocol.}
We evaluate eight capabilities using Appendix
Table~\ref{tab:capabilities-benchmarks}. For each benchmark, 10\% of non-test
data is used for profiling, requirement inference, and allocation updates; the
held-out split is reserved for final reporting. Generation uses temperature $0$
with the same prompt template across methods. We report mean $\pm$ standard
deviation over three seeds.

\textbf{Baselines.}
For each capability, we build a held-out-filtered distillation pool from its
benchmark prompts. For multi-dataset capabilities, we use balanced dataset-level
sampling. The six main baselines combine three teacher signals---final response
(Resp), chain-of-thought plus response (CoT), and token-level logits (Logit)---
with two objectives, SFT and KD. Each baseline is one-hot at the capability
level, spending the full budget on the target-capability pool.

\textbf{ReAD implementation.}
ReAD runs for $20$ decision steps, with bandit updates, proxy refreshes, and data
resampling every $1.0$M tokens in the $20$M setting and every $7.5$M tokens in
the $150$M setting. Reward regressors are two-layer MLPs trained on the 10\%
profiling split plus logged transitions, yielding $3.2$k--$4.9$k examples per
task. Proxy refresh adds 3.12\% wall-clock overhead, the full online loop adds
5.63\%, and the reusable requirement identifier costs 3.79 GPU-hours to train
once.

\section{Experiments}
\label{sec:exp}

We study four questions: \textbf{RQ1} utility gains under the same budget;
\textbf{RQ2} reduced waste and spillover; \textbf{RQ3} the value of adaptive
bandit scheduling over static or greedy alternatives; and \textbf{RQ4}
transfer beyond capability-aligned benchmarks.

\begin{table*}[t]
\centering
\tiny
\setlength\tabcolsep{1pt}
\caption{Capability profile under a 20M-token budget on Llama. ReAD achieves the strongest average profile while using the same student initialization and token budget as all baselines.}
\begin{tabular}{lcccccccc}
\toprule
\textbf{Method} &
\textbf{General} &
\textbf{Steer.} &
\textbf{Reason.} &
\textbf{Math} &
\textbf{Code} &
\textbf{Tool} &
\textbf{LCU} &
\textbf{Multi.} \\
\midrule
Teacher only &
$48.13\pm0.35$ & $89.98\pm0.40$ & $10.51\pm0.12$ & $48.34\pm0.55$ &
$88.40\pm0.70$ & $31.90\pm0.45$ & $36.20\pm0.40$ & $91.10\pm0.55$ \\
Student only &
$31.09\pm0.40$ & $49.22\pm0.55$ & $8.72\pm0.15$ & $15.56\pm0.60$ &
$72.60\pm0.95$ & $25.83\pm0.60$ & $30.40\pm0.45$ & $68.90\pm0.80$ \\
\midrule
Resp-SFT & $33.28\pm0.47$ & $52.84\pm0.71$ & $9.05\pm0.18$ & $17.12\pm0.83$ & $74.06\pm1.09$ & $26.92\pm0.78$ & $31.28\pm0.51$ & $70.02\pm0.93$ \\
CoT-SFT & $33.02\pm0.52$ & $52.16\pm0.76$ & $9.46\pm0.14$ & $18.58\pm0.77$ & $73.82\pm1.02$ & $26.81\pm0.74$ & $31.19\pm0.57$ & $69.78\pm0.88$ \\
Logit-SFT & $33.61\pm0.44$ & $53.04\pm0.69$ & $9.12\pm0.17$ & $17.63\pm0.79$ & $74.48\pm1.15$ & $27.03\pm0.70$ & $31.42\pm0.53$ & $70.21\pm0.90$ \\
Resp-KD & $33.83\pm0.50$ & $53.62\pm0.73$ & $9.18\pm0.16$ & $17.84\pm0.74$ & $74.23\pm1.07$ & $27.18\pm0.76$ & $31.58\pm0.55$ & $70.48\pm0.95$ \\
CoT-KD & $33.46\pm0.46$ & $52.98\pm0.78$ & $9.62\pm0.13$ & $19.34\pm0.72$ & $74.03\pm1.10$ & $27.06\pm0.72$ & $31.47\pm0.58$ & $70.08\pm0.89$ \\
Logit-KD & $34.02\pm0.43$ & $54.21\pm0.66$ & $9.25\pm0.15$ & $18.20\pm0.76$ & $74.62\pm1.12$ & $27.41\pm0.69$ & $31.72\pm0.50$ & $70.81\pm0.91$ \\
\midrule
\textbf{ReAD} &
$\mathbf{35.02\pm0.41}$ & $\mathbf{58.03\pm0.61}$ & $\mathbf{10.01\pm0.12}$ & $\mathbf{21.06\pm0.68}$ &
$\mathbf{77.06\pm0.97}$ & $\mathbf{29.82\pm0.64}$ & $\mathbf{33.58\pm0.49}$ & $\mathbf{73.04\pm0.84}$ \\
\bottomrule
\end{tabular}
\label{tab:llama_20m}
\end{table*}

\begin{figure}[t]
  \centering
  \includegraphics[width=\columnwidth]{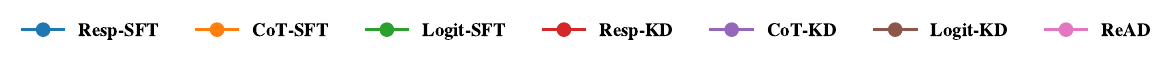}
  \vspace{-1.5em}

  \begin{subfigure}[t]{0.24\columnwidth}
    \centering
    \includegraphics[width=\linewidth]{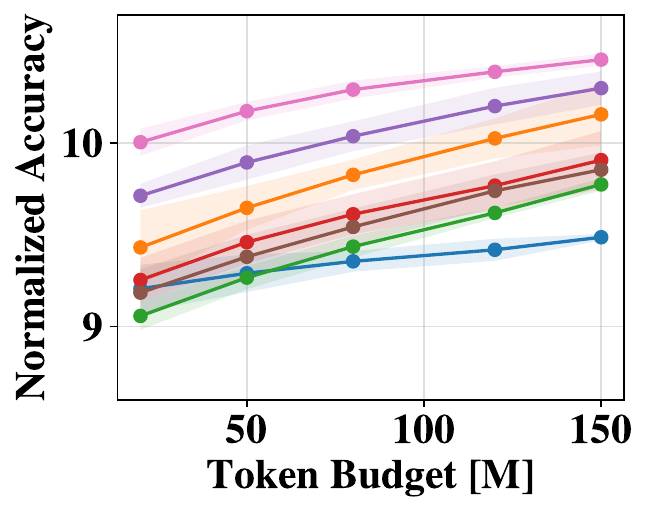}
    \caption{Reasoning}
    \label{fig:r2_reason}
  \end{subfigure}
  \hfill
  \begin{subfigure}[t]{0.24\columnwidth}
    \centering
    \includegraphics[width=\linewidth]{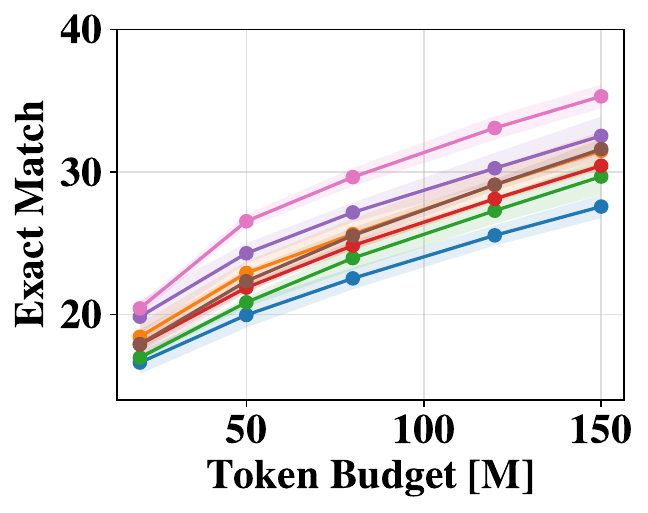}
    \caption{Math}
    \label{fig:r2_math}
  \end{subfigure}
  \hfill
  \begin{subfigure}[t]{0.24\columnwidth}
    \centering
    \includegraphics[width=\linewidth]{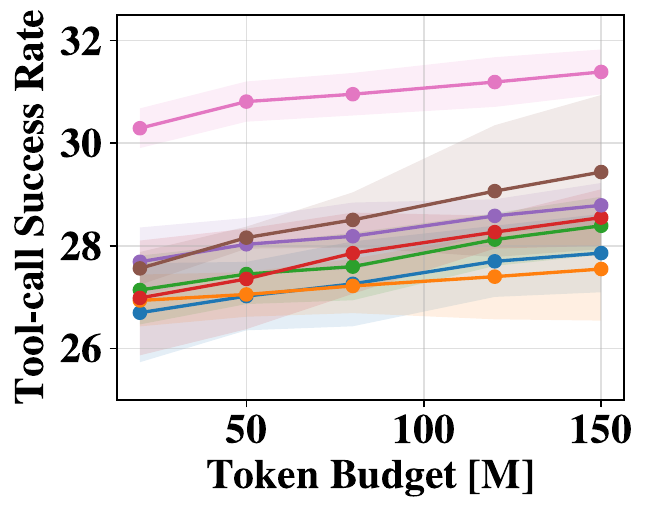}
    \caption{Tool Use}
    \label{fig:r2_tool}
  \end{subfigure}
  \hfill
  \begin{subfigure}[t]{0.24\columnwidth}
    \centering
    \includegraphics[width=\linewidth]{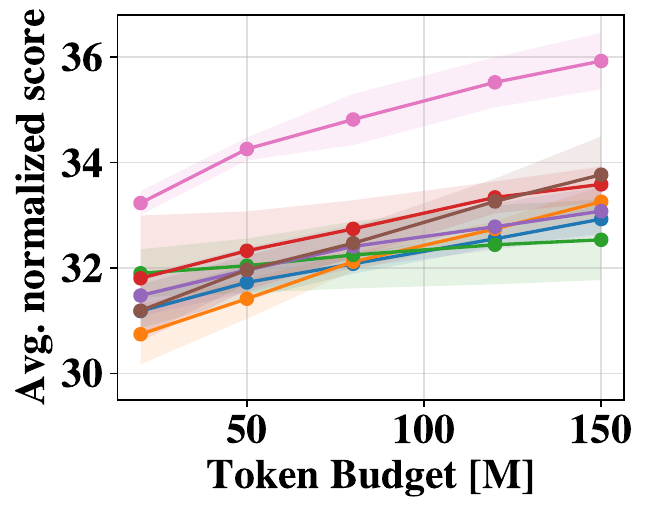}
    \caption{LCU}
    \label{fig:r2_lcu}
  \end{subfigure}

  \caption{Performance versus token budget for four bottleneck capabilities. Curves show mean $\pm$ standard deviation over three seeds.}
  \label{fig:rq2_budget}
  \vspace{-0.8em}
\end{figure}

\subsection{Results and Analysis}
\label{subsec:exp_results}

\paragraph{RQ1: Utility under a fixed budget.}
Table~\ref{tab:llama_20m} shows that ReAD achieves the strongest capability
profile at $20$M tokens, consistently outperforming all six baselines. The gains
are especially pronounced on bottleneck capabilities such as Steerability, Math,
Tool Use, and LCU. The $150$M Llama results and Qwen results in
Appendix~\ref{sec:app_experiment} show the same ordering, suggesting that the
improvement is not tied to a single teacher--student pair.

\paragraph{RQ2: Budget efficiency.}
Figure~\ref{fig:rq2_budget} reports budget-scaling curves for Reasoning, Math,
Tool Use, and LCU. ReAD achieves higher performance at the same budget and
maintains the strongest scaling trend, indicating that adaptive allocation
reduces wasted tokens on low-marginal-return updates.

\begin{table}[t]
  \centering
  \captionsetup{font=small, width=0.96\linewidth, skip=3pt}
  \caption{\textbf{Budget-allocation diagnostics.} Panel (a) reports the 150M-token gain--spillover trade-off computed from the
capability transfer matrix: ReAD preserves near-best on-target gain while
substantially reducing off-target degradation and negative transfer.
Panel (b) compares matched-budget allocation schedules under the same generator,
showing that adaptive bandit allocation outperforms static and greedy schedules.}
  \label{tab:budget_diagnostics}
  \vspace{-0.4em}

  \footnotesize
  \setlength{\tabcolsep}{2.2pt}
  \renewcommand{\arraystretch}{0.94}

  \begin{subtable}[t]{0.55\linewidth}
    \centering
    \caption*{\textbf{(a) Harmful transfer at 150M}}
    \begin{tabular}{@{}lccc@{}}
      \toprule
      \textbf{Method} & \textbf{On-tgt $\uparrow$} & \textbf{Off-tgt $\Delta\uparrow$} & \textbf{Neg. $\downarrow$} \\
      \midrule
      Single-cap. KD
        & $\mathbf{4.79{\pm}0.29}$
        & $-1.57{\pm}0.21$
        & $0.43{\pm}0.03$ \\
      Task-static mix
        & $4.61{\pm}0.25$
        & $-0.87{\pm}0.17$
        & $0.29{\pm}0.03$ \\
      Greedy one-step
        & $4.72{\pm}0.23$ 
        & $-0.94{\pm}0.19$ 
        & $0.31{\pm}0.03$ 
        \\
      Grid-searched
        & $4.66{\pm}0.24$ 
        & $-0.71{\pm}0.16$ 
        & $0.25{\pm}0.03$ 
        \\
      \textbf{ReAD}
        & $4.79{\pm}0.24$
        & $\mathbf{-0.38{\pm}0.12}$
        & $\mathbf{0.17{\pm}0.02}$ \\
      \bottomrule
    \end{tabular}
  \end{subtable}
  \hfill
  \begin{subtable}[t]{0.42\linewidth}
    \centering
    \caption*{\textbf{(b) Allocation schedules}}
    \begin{tabular}{@{}lcc@{}}
      \toprule
      \textbf{Method} & \textbf{20M $\uparrow$} & \textbf{150M $\uparrow$} \\
      \midrule
      Uniform static
        & $60.94{\pm}0.48$
        & $64.73{\pm}0.36$ \\
      Task-static mix
        & $61.71{\pm}0.43$
        & $65.52{\pm}0.33$ \\
      Greedy one-step
        & $62.03{\pm}0.41$
        & $66.01{\pm}0.30$ \\
      Grid-searched
        & $61.88{\pm}0.42$
        & $65.81{\pm}0.31$ \\
      \textbf{ReAD}
        & $\mathbf{63.04{\pm}0.39}$
        & $\mathbf{67.21{\pm}0.28}$ \\
      \bottomrule
    \end{tabular}
  \end{subtable}

  \vspace{-0.8em}
\end{table}

\paragraph{RQ2: Gain--spillover trade-off.}
Table~\ref{tab:budget_diagnostics}(a) shows that ReAD improves the
gain--spillover trade-off. Single-capability KD achieves strong target gains but
causes larger off-target degradation, while static mixing reduces spillover at
the cost of weaker target improvement. ReAD better balances the two: it keeps
strong on-target gains while suppressing negative transfer across
non-target capabilities.

\paragraph{RQ3: Adaptive scheduling.}
Table~\ref{tab:budget_diagnostics}(b) compares ReAD with stronger allocation
baselines under the same generator and token budgets. Static schedules cannot
adapt once marginal returns and spillover change during training, while the
greedy heuristic ignores uncertainty and remaining budget. ReAD outperforms
both, with a larger margin at $150$M where saturation and spillover are stronger.

\begin{wrapfigure}[9]{r}{0.42\textwidth}
  \vspace{-1.2em}
  \centering
  \includegraphics[width=0.40\textwidth]{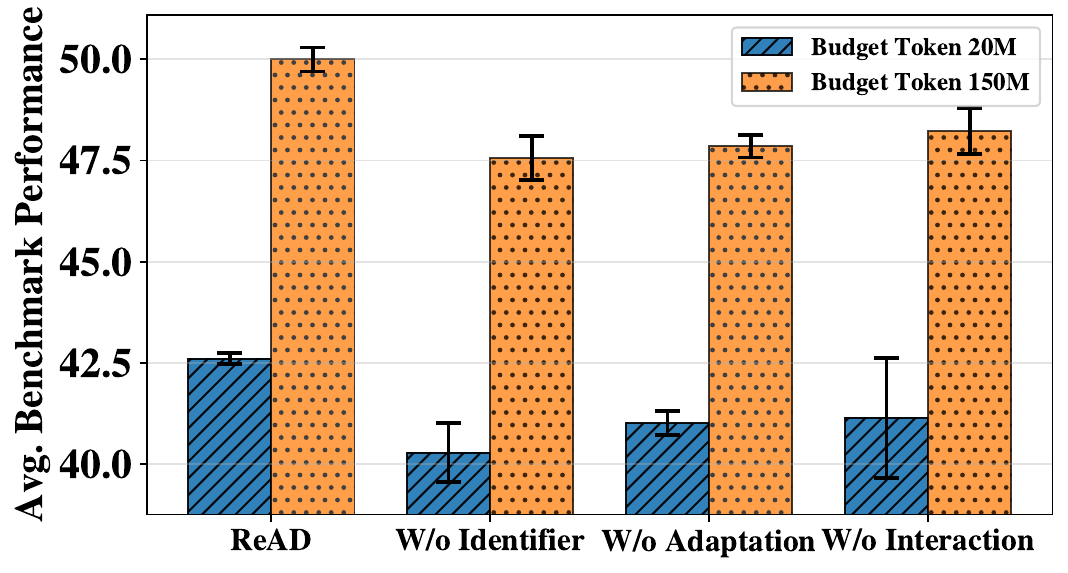}
  \caption{Ablation of ReAD components.}
  \label{fig:rq3_ablation}
  \vspace{-0.8em}
\end{wrapfigure}

\paragraph{Component ablation.}
We ablate ReAD by removing one component while fixing the teacher, student,
prompt pools, training recipe, and budget. Removing requirement identification
makes the allocation uniform; removing adaptation uses a fixed allocation; and
removing interaction awareness drops the spillover penalty.
Figure~\ref{fig:rq3_ablation} shows that identification and adaptation drive the
largest gains, while interaction modeling provides additional improvement by
reducing harmful transfer.

\paragraph{Local proxy validation.}
The proxy validation supports ReAD's use of local action ranking. Predicted
next-step gains align with observed gains, especially for small budget moves,
which matches ReAD's stepwise allocation regime. The requirement vectors are
stable across runs, and the reward regressors are accurate enough to guide the
uncertainty-aware bandit updates.

\paragraph{RQ4: Transfer beyond aligned benchmarks.}
Finally, we evaluate ReAD on held-out XSTest~\cite{rottger2024xstest} using the same requirement
identifier and allocation pipeline without adding a new capability head. ReAD
improves both safe-refusal and unsafe-refusal metrics over the best SFT/KD
baseline; full results are reported in Appendix Table~\ref{tab:xstest}.
\section{Conclusions}
\label{sec:con}

We study capability distillation under a fixed token budget and find systematic cross-transfer and diminishing returns, often with increasing harmful spillover on task-critical behaviors. We propose ReAD, a reinforcement-guided framework that identifies task-essential capabilities and adaptively allocates budget across interdependent capabilities via on-the-fly supervision and an uncertainty-aware contextual bandit. Experiments show ReAD improves downstream utility under the same budget while reducing waste.

\bibliographystyle{plain}
\bibliography{ref}

\appendix

\newpage
\section*{Technical Appendices and Supplementary Material}

\section{Benchmark Datasets}\label{sec:bench_datasets}

Table~\ref{tab:capabilities-benchmarks} summarizes the benchmark suite used to
evaluate each capability. For capabilities associated with multiple benchmarks,
we report the average score across the listed benchmarks.

\begin{table}[h]
  \centering
  \caption{Benchmark suite for evaluating LLM capabilities.}
  \label{tab:capabilities-benchmarks}
  \small
  \setlength{\tabcolsep}{4pt}
  \renewcommand{\arraystretch}{0.95}

  \begin{tabularx}{\linewidth}{@{}
    >{\raggedright\arraybackslash}p{0.20\linewidth}
    >{\raggedright\arraybackslash}p{0.36\linewidth}
    >{\raggedright\arraybackslash}X
  @{}}
    \toprule
    \textbf{Capability} & \textbf{Benchmarks} & \textbf{Metric} \\
    \midrule
    General      & MMLU-Pro, BIG       & Accuracy \\
    Steerability & IFEval, FollowEval  & Instruction-following accuracy \\
    Reasoning    & GPQA, ARC-Challenge & Normalized accuracy \\
    Math         & MATH Hard, GSM8K    & Exact match \\
    Code         & HumanEval, MBPP     & Pass@1 \\
    Tool Use     & BFCL V4, ToolBench  & Tool-call success rate \\
    LCU          & LongBench           & Average normalized score \\
    Multilingual & XCOPA, MGSM         & Macro accuracy \\
    \bottomrule
  \end{tabularx}
\end{table}

\begin{itemize}
    \item \textbf{MMLU-Pro}~\cite{wang2024mmluprorobustchallengingmultitask} tests broad knowledge and problem-solving across many academic subjects using multiple-choice questions. 
    \item \textbf{BIG-bench (BIG)}~\cite{suzgun2022challenging} is a large collection of diverse tasks designed to probe general reasoning and knowledge beyond standard exams.
    \item \textbf{IFEval}~\cite{zhou2023instructionfollowing} measures instruction following by checking whether model outputs satisfy explicit, verifiable constraints in the prompt.
    \item \textbf{FollowEval}~\cite{jing2023followeval} evaluates how well a model follows user instructions under varied formats and constraint types, complementing IFEval with broader instruction patterns.
    \item \textbf{GPQA}~\cite{rein2024gpqa} is a graduate-level question answering benchmark that stresses hard scientific reasoning with high-quality, expert-written questions.
    \item \textbf{ARC-Challenge}~\cite{Clark2018ThinkYH} contains difficult grade-school science questions that require multi-step reasoning and commonsense beyond keyword matching.
    \item \textbf{MATH Hard}~\cite{hendrycksmath2021} is a challenging subset of competition-style math problems and is scored by exact final-answer match.
    \item \textbf{GSM8K}~\cite{cobbe2021training} focuses on grade-school math word problems that require intermediate reasoning steps, also scored by exact final-answer match.
    \item \textbf{HumanEval}~\cite{chen2021codex} evaluates code generation via functional correctness on hand-written programming problems, reported as Pass@1.
    \item \textbf{MBPP}~\cite{austin2021program} is a Python programming benchmark with short programming tasks and unit tests, also evaluated by Pass@1.
    \item \textbf{BFCL V4}~\cite{berkeley-function-calling-leaderboard} measures tool-use ability by checking whether the model can produce correct function calls under realistic tool specifications.
    \item \textbf{ToolBench}~\cite{xu2023tool} evaluates tool use in more complex, multi-tool scenarios and is scored by tool-call success rate.
    \item \textbf{LongBench}~\cite{bai2024longbench} evaluates long-context understanding across multiple long-input tasks and reports an average normalized score.
    \item \textbf{XCOPA}~\cite{ponti2020xcopa} tests multilingual causal commonsense reasoning across languages and is reported using macro accuracy.
    \item \textbf{MGSM}~\cite{shi2022language} is a multilingual version of grade-school math and is evaluated by macro accuracy across languages.
\end{itemize}

\section{Empirical Study Setting}\label{sec:app_em}

We study capability interactions under a fixed distillation budget by distilling a student model toward one target capability at a time and then evaluating the resulting student on \emph{all} capabilities in Table~\ref{tab:capabilities-benchmarks}. 
For each capability, we build a capability-specific distillation pool from the corresponding benchmark prompts, excluding any held-out evaluation examples to prevent leakage. 
We repeat this process under multiple token budgets (e.g., 20M, 80M, 150M tokens) to form a budget-dependent capability transfer matrix, where each entry is the score change on a destination capability when distilling toward a source capability at budget $B$. 
We use deterministic decoding for generation benchmarks (temperature $=0$) with a fixed prompt template across methods, and we compute each capability score as the average over its associated benchmarks in the table. We instantiate a simple but representative set of capability distillation baselines by combining three supervision forms that expose different levels of teacher knowledge---\textit{Resp} (final answer only), \textit{CoT} (rationale plus final answer), and \textit{Logit} (token-level soft targets)—with two standard objectives, \textit{SFT} and \textit{KD}, yielding six baselines in total.

\section{Additional Theoretical Analysis}
\textbf{Uncertainty-aware UCB allocation}. Here, we conduct an uncertainty-aware UCB allocation analysis showing that action-level regret is controlled by predictive uncertainty. At interval $t$, ReAD maintains a bandit context
\begin{equation}
\mathbf{x}_t := [\,\mathbf{r}_\tau;\ \mathbf{s}(S_{t});\ b_t\,],
\label{eq:context_short}
\end{equation}
where $\mathbf{r}_\tau\in\Delta$ is the task requirement vector, $\mathbf{s}(S_{t})\in\mathbb{R}^{|\mathcal{C}|}$ is the probe-defined capability profile of the current student, and $b_t$ is the remaining budget. Let $\mathcal{A}(\tau)$ be the discrete candidate action set of allocation vectors. Define the conditional expected proxy reward
\begin{equation}
\bar R(\mathbf{x},\mathbf{w}) := \mathbb{E}\!\left[\widehat{R}_t\mid \mathbf{x}_t=\mathbf{x},\ \mathbf{w}_t=\mathbf{w}\right],
\label{eq:expected_reward_short}
\end{equation}
where $\widehat{R}_t$ is the proxy reward used as shown in Eq.~\eqref{eq:reward}.

\begin{assumption}[Calibrated uncertainty]\label{ass:calibrated_uncertainty}
    For a confidence level $\delta\in(0,1)$, with probability at least $1-\delta$, for all encountered $(\mathbf{x}_t,\mathbf{w})$,
\begin{equation}
\big|\mu(\mathbf{x}_t,\mathbf{w})-\bar R(\mathbf{x}_t,\mathbf{w})\big| \le \sigma(\mathbf{x}_t,\mathbf{w}),
\label{eq:calibration_short}
\end{equation}
where $\mu$ and $\sigma$ are the ensemble mean and standard deviation defined from the reward regressors.
\end{assumption}

Since ReAD selects allocations by the UCB rule in Eq.~\eqref{eq:ucb}, we can compare its chosen action $\mathbf{w}_t$ to the optimal action $\mathbf{w}_t^\star$ under the same context and bound the resulting gap in conditional expected reward using the calibrated uncertainty condition in Assumption~\ref{ass:calibrated_uncertainty}, which leads to the following instantaneous regret bound.

\begin{proposition}[Instantaneous regret bound]
    Let $\mathbf{w}_t^\star\in\arg\max_{\mathbf{w}\in\mathcal{A}(\tau)} \bar R(\mathbf{x}_t,\mathbf{w})$. Under Assumption~\ref{ass:calibrated_uncertainty},
\begin{equation}
\bar R(\mathbf{x}_t,\mathbf{w}_t^\star)-\bar R(\mathbf{x}_t,\mathbf{w}_t)\ \le\ 2\kappa\,\sigma(\mathbf{x}_t,\mathbf{w}_t),
\label{eq:inst_regret_short}
\end{equation}
and consequently,
\begin{equation}
\sum_{t=1}^T \Big(\bar R(\mathbf{x}_t,\mathbf{w}_t^\star)-\bar R(\mathbf{x}_t,\mathbf{w}_t)\Big)
\ \le\ 2\kappa\sum_{t=1}^T \sigma(\mathbf{x}_t,\mathbf{w}_t).
\label{eq:cum_regret_short}
\end{equation}
\end{proposition}

\section{Additional Experiments}\label{sec:app_experiment}

\begin{table*}[t]
\centering
\tiny
\setlength\tabcolsep{1pt}
\caption{Capability profile of different methods under a 150M-token budget (LLama). ReAD consistently outperforms all six baselines.}
\begin{tabular}{lcccccccc}
\toprule
\textbf{Method} &
\textbf{General} &
\textbf{Steerability} &
\textbf{Reasoning} &
\textbf{Math} &
\textbf{Code} &
\textbf{Tool Use} &
\textbf{LCU} &
\textbf{Multilingual} \\
\midrule
Teacher only &
$48.13\pm0.35$ &
$89.98\pm0.40$ &
$10.51\pm0.12$ &
$48.34\pm0.55$ &
$88.40\pm0.70$ &
$31.90\pm0.45$ &
$36.20\pm0.40$ &
$91.10\pm0.55$ \\
Student only &
$31.09\pm0.40$ &
$49.22\pm0.55$ &
$8.72\pm0.15$ &
$15.56\pm0.60$ &
$72.60\pm0.95$ &
$25.83\pm0.60$ &
$30.40\pm0.45$ &
$68.90\pm0.80$ \\
\midrule
Resp-SFT &
$38.41\pm0.58$ &
$68.73\pm0.94$ &
$9.55\pm0.14$ &
$28.72\pm1.08$ &
$79.18\pm1.26$ &
$28.53\pm0.82$ &
$32.69\pm0.73$ &
$78.23\pm1.12$ \\
CoT-SFT &
$38.02\pm0.61$ &
$67.92\pm0.89$ &
$10.14\pm0.12$ &
$31.58\pm0.97$ &
$78.86\pm1.33$ &
$28.41\pm0.79$ &
$32.83\pm0.68$ &
$77.77\pm1.06$ \\
Logit-SFT &
$39.03\pm0.56$ &
$69.38\pm0.86$ &
$9.70\pm0.13$ &
$29.61\pm1.02$ &
$79.69\pm1.21$ &
$28.71\pm0.77$ &
$33.01\pm0.70$ &
$78.63\pm1.09$ \\
Resp-KD &
$39.62\pm0.53$ &
$70.98\pm0.84$ &
$9.82\pm0.12$ &
$30.21\pm0.98$ &
$80.09\pm1.18$ &
$29.09\pm0.74$ &
$33.42\pm0.66$ &
$79.18\pm1.03$ \\
CoT-KD &
$39.12\pm0.55$ &
$69.84\pm0.88$ &
$10.26\pm0.11$ &
$33.02\pm0.92$ &
$79.81\pm1.24$ &
$29.01\pm0.71$ &
$33.19\pm0.69$ &
$78.71\pm1.01$ \\
Logit-KD &
$40.18\pm0.51$ &
$72.63\pm0.79$ &
$9.95\pm0.12$ &
$31.36\pm0.95$ &
$80.57\pm1.16$ &
$29.48\pm0.69$ &
$33.73\pm0.64$ &
$79.97\pm0.97$ \\
\midrule
\textbf{ReAD} &
$\mathbf{41.80\pm0.45}$ &
$\mathbf{78.20\pm0.70}$ &
$\mathbf{10.40\pm0.10}$ &
$\mathbf{35.20\pm0.80}$ &
$\mathbf{83.90\pm1.00}$ &
$\mathbf{31.30\pm0.55}$ &
$\mathbf{35.60\pm0.55}$ &
$\mathbf{84.10\pm0.90}$ \\
\bottomrule
\end{tabular}
\label{tab:llama_150m}
\end{table*}

\begin{table*}[t]
\centering
\tiny
\setlength\tabcolsep{1pt}
\caption{Capability profile of different methods under a 20M-token budget (Qwen). ReAD consistently outperforms all six baselines.}
\begin{tabular}{lcccccccc}
\toprule
\textbf{Method} &
\textbf{General} &
\textbf{Steerability} &
\textbf{Reasoning} &
\textbf{Math} &
\textbf{Code} &
\textbf{Tool Use} &
\textbf{LCU} &
\textbf{Multilingual} \\
\midrule
Teacher only &
$55.27 \pm 0.46$ &
$92.36 \pm 0.52$ &
$14.28 \pm 0.18$ &
$55.74 \pm 0.78$ &
$90.12 \pm 0.88$ &
$36.08 \pm 0.63$ &
$40.23 \pm 0.57$ &
$92.34 \pm 0.66$ \\
Student only &
$40.07 \pm 0.61$ &
$70.18 \pm 0.74$ &
$11.47 \pm 0.22$ &
$28.26 \pm 0.95$ &
$81.96 \pm 1.10$ &
$30.06 \pm 0.82$ &
$34.17 \pm 0.71$ &
$83.14 \pm 0.93$ \\
\midrule
Resp-SFT &
$42.18 \pm 0.86$ &
$72.92 \pm 0.98$ &
$11.62 \pm 0.24$ &
$31.42 \pm 1.35$ &
$83.88 \pm 1.42$ &
$30.96 \pm 0.97$ &
$35.04 \pm 0.88$ &
$84.36 \pm 1.21$ \\
CoT-SFT &
$41.93 \pm 0.90$ &
$72.41 \pm 1.02$ &
$12.06 \pm 0.21$ &
$33.12 \pm 1.28$ &
$83.54 \pm 1.50$ &
$30.74 \pm 0.95$ &
$34.92 \pm 0.91$ &
$84.09 \pm 1.18$ \\
Logit-SFT &
$42.57 \pm 0.82$ &
$73.56 \pm 0.95$ &
$11.81 \pm 0.23$ &
$32.24 \pm 1.22$ &
$84.06 \pm 1.37$ &
$31.13 \pm 0.93$ &
$35.21 \pm 0.86$ &
$84.58 \pm 1.16$ \\
Resp-KD &
$42.79 \pm 0.80$ &
$74.03 \pm 0.92$ &
$11.93 \pm 0.22$ &
$32.68 \pm 1.18$ &
$84.22 \pm 1.30$ &
$31.42 \pm 0.90$ &
$35.44 \pm 0.83$ &
$84.92 \pm 1.10$ \\
CoT-KD &
$42.44 \pm 0.83$ &
$73.37 \pm 0.94$ &
$12.31 \pm 0.19$ &
$34.08 \pm 1.12$ &
$83.91 \pm 1.28$ &
$31.02 \pm 0.88$ &
$35.29 \pm 0.85$ &
$84.76 \pm 1.08$ \\
Logit-KD &
$43.06 \pm 0.78$ &
$74.74 \pm 0.88$ &
$12.04 \pm 0.21$ &
$33.07 \pm 1.10$ &
$84.55 \pm 1.22$ &
$31.77 \pm 0.86$ &
$35.73 \pm 0.80$ &
$85.27 \pm 1.02$ \\
\midrule
\textbf{ReAD} &
$\mathbf{44.63 \pm 0.73}$ &
$\mathbf{78.57 \pm 0.86}$ &
$\mathbf{12.86 \pm 0.17}$ &
$\mathbf{36.24 \pm 1.05}$ &
$\mathbf{86.38 \pm 1.18}$ &
$\mathbf{33.41 \pm 0.79}$ &
$\mathbf{36.98 \pm 0.77}$ &
$\mathbf{87.23 \pm 0.95}$ \\
\bottomrule
\end{tabular}
\label{tab:qwen_20m}
\end{table*}

\begin{table*}[t]
\centering
\tiny
\setlength\tabcolsep{1pt}
\caption{Capability profile of different methods under a 150M-token budget (Qwen). ReAD consistently outperforms all six baselines.}
\begin{tabular}{lcccccccc}
\toprule
\textbf{Method} &
\textbf{General} &
\textbf{Steerability} &
\textbf{Reasoning} &
\textbf{Math} &
\textbf{Code} &
\textbf{Tool Use} &
\textbf{LCU} &
\textbf{Multilingual} \\
\midrule
Teacher only &
$55.27 \pm 0.46$ &
$92.36 \pm 0.52$ &
$14.28 \pm 0.18$ &
$55.74 \pm 0.78$ &
$90.12 \pm 0.88$ &
$36.08 \pm 0.63$ &
$40.23 \pm 0.57$ &
$92.34 \pm 0.66$ \\
Student only &
$40.07 \pm 0.61$ &
$70.18 \pm 0.74$ &
$11.47 \pm 0.22$ &
$28.26 \pm 0.95$ &
$81.96 \pm 1.10$ &
$30.06 \pm 0.82$ &
$34.17 \pm 0.71$ &
$83.14 \pm 0.93$ \\
\midrule
Resp-SFT &
$48.02 \pm 0.92$ &
$82.07 \pm 1.04$ &
$13.18 \pm 0.20$ &
$41.26 \pm 1.38$ &
$86.92 \pm 1.40$ &
$33.29 \pm 0.96$ &
$37.42 \pm 0.88$ &
$88.07 \pm 1.20$ \\
CoT-SFT &
$48.29 \pm 0.88$ &
$82.61 \pm 0.98$ &
$13.72 \pm 0.18$ &
$44.43 \pm 1.30$ &
$86.58 \pm 1.45$ &
$33.05 \pm 0.94$ &
$37.28 \pm 0.90$ &
$87.96 \pm 1.18$ \\
Logit-SFT &
$48.46 \pm 0.86$ &
$83.14 \pm 0.95$ &
$13.39 \pm 0.19$ &
$42.68 \pm 1.26$ &
$87.05 \pm 1.32$ &
$33.47 \pm 0.92$ &
$37.63 \pm 0.86$ &
$88.31 \pm 1.12$ \\
Resp-KD &
$48.78 \pm 0.84$ &
$83.76 \pm 0.90$ &
$13.51 \pm 0.18$ &
$43.62 \pm 1.22$ &
$87.37 \pm 1.26$ &
$33.92 \pm 0.88$ &
$38.02 \pm 0.82$ &
$88.64 \pm 1.06$ \\
CoT-KD &
$49.05 \pm 0.82$ &
$84.28 \pm 0.88$ &
$13.87 \pm 0.17$ &
$45.17 \pm 1.18$ &
$87.11 \pm 1.20$ &
$33.66 \pm 0.86$ &
$37.91 \pm 0.84$ &
$88.53 \pm 1.02$ \\
Logit-KD &
$49.37 \pm 0.80$ &
$85.19 \pm 0.84$ &
$13.62 \pm 0.18$ &
$44.07 \pm 1.16$ &
$87.63 \pm 1.18$ &
$34.26 \pm 0.84$ &
$38.33 \pm 0.78$ &
$88.92 \pm 0.98$ \\
\midrule
\textbf{ReAD} &
$\mathbf{51.53 \pm 0.76}$ &
$\mathbf{88.04 \pm 0.90}$ &
$\mathbf{14.03 \pm 0.15}$ &
$\mathbf{47.06 \pm 1.05}$ &
$\mathbf{89.04 \pm 1.12}$ &
$\mathbf{35.52 \pm 0.78}$ &
$\mathbf{39.47 \pm 0.74}$ &
$\mathbf{90.54 \pm 0.92}$ \\
\bottomrule
\end{tabular}
\label{tab:qwen_150m}
\end{table*}

\textbf{Additional evidence showing that ReAD improves capability distillation under a fixed budget}. We compare ReAD with six capability distillation baselines. Tables~\ref{tab:llama_20m} (in the main text) and~\ref{tab:llama_150m} report the Llama results at $B{=}20$M and $B{=}150$M, and Tables~\ref{tab:qwen_20m} and~\ref{tab:qwen_150m} report the corresponding Qwen results. In both budgets and for both model families, ReAD achieves the best overall capability profile and outperforms all six baselines on every capability, with the largest improvements appearing on bottleneck capabilities such as Steerability and Math while still maintaining strong gains on Code, Tool Use, and LCU. Scaling the budget from 20M to 150M strengthens all methods, but ReAD benefits more consistently, indicating that its allocation policy converts additional tokens into broader capability improvements rather than concentrating gains on a narrow subset of skills. Across model families, Qwen starts from a stronger student and teacher on most capabilities and also reaches higher absolute performance after distillation, but the relative ordering between methods is unchanged and ReAD remains the top performer, suggesting that ReAD is robust to the choice of teacher--student backbone and that its advantage is not tied to a specific LLM family.

\begin{wrapfigure}{r}{0.45\columnwidth}
  \centering
  \includegraphics[width=0.45\columnwidth]{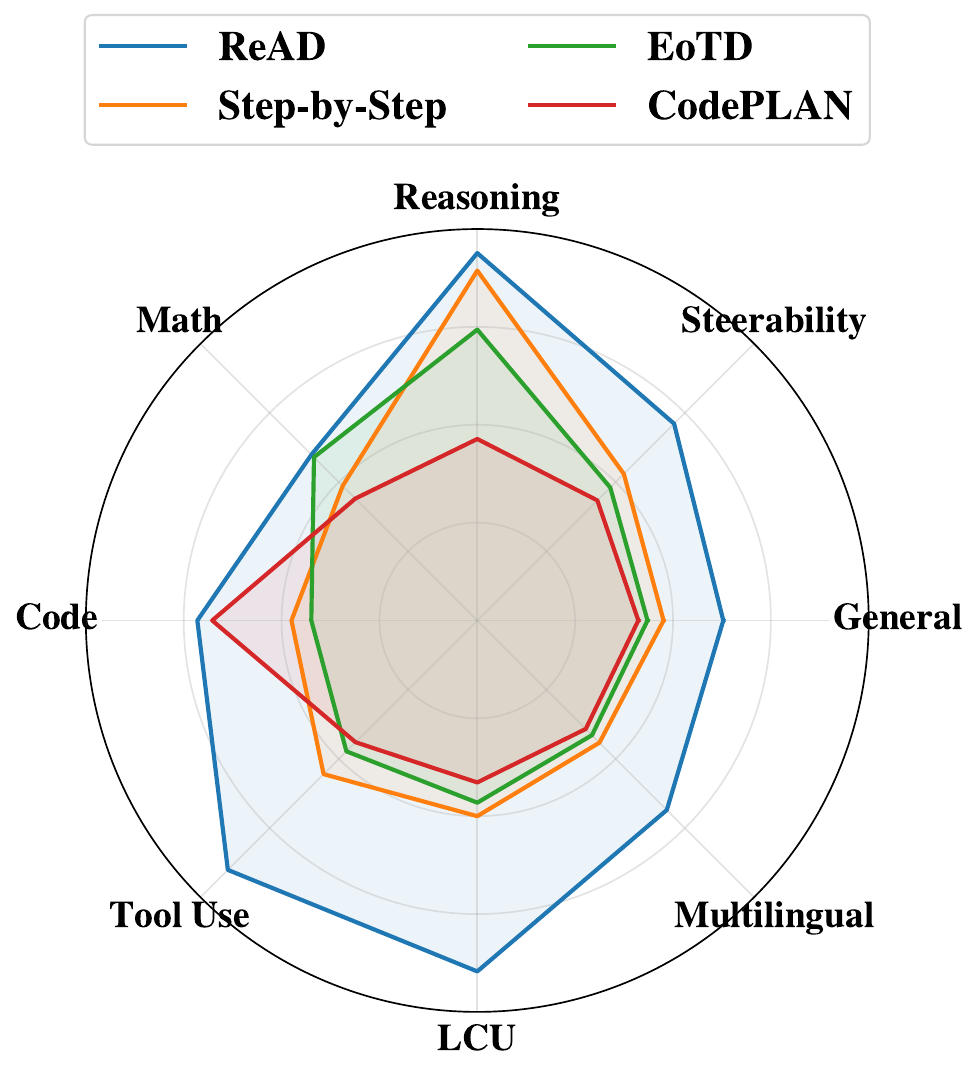}
  \caption{ReAD beats SOTA baselines that are specifically designed for reasoning, code, and math capabilities consistently.}
  \label{fig:radar}
\end{wrapfigure}

\textbf{ReAD beats SOTA baselines.}
To further validate ReAD against SOTA capability distillation baselines that are specifically designed for certain 
capabilities, we include three representative methods that are commonly used for specializing student models, one per capability: step-by-step distillation~\cite{hsieh2023stepbystep} for reasoning, Equation-of-Thought Distillation (EoTD)~\cite{zhu2024distilling} for math, and CodePLAN~\cite{sun2024enhancing} for code. For a fair comparison under our setting, we adapt each baseline to our LLaMa teacher-student pair and token budget as $150$M, and we distill the student using only the corresponding capability pool. We report the resulting capability profiles on all eight benchmarks using the same evaluation protocol as the main experiments in Figure~\ref{fig:radar}. The results show that these capability-targeted baselines mainly concentrate gains on their target capability and exhibit weaker transfer to other capabilities, producing an imbalanced profile that matches the interaction patterns observed in our empirical study. In contrast, ReAD achieves a consistently stronger and more balanced profile across capabilities, and it is never worse than these baselines on its target capability, while improving substantially on non-target capabilities.


\textbf{Held-out Downstream Transfer}. To test whether ReAD transfers beyond capability-aligned benchmark tasks, we
evaluate on XSTest using the same requirement identifier and allocation pipeline
without adding a new capability head. Lower safe-refusal and higher
unsafe-refusal scores are better.

\begin{table}[h]
  \centering
  \caption{
  Held-out XSTest evaluation. ``Best baseline'' denotes the best of the six
  SFT/KD baselines for each metric.
  }
  \label{tab:xstest}
  \scriptsize
  \setlength{\tabcolsep}{4pt}
  \renewcommand{\arraystretch}{0.95}

  \begin{tabular}{@{}lcccc@{}}
    \toprule
    \textbf{Method}
      & \textbf{20M safe $\downarrow$}
      & \textbf{20M unsafe $\uparrow$}
      & \textbf{150M safe $\downarrow$}
      & \textbf{150M unsafe $\uparrow$} \\
    \midrule
    Best baseline
      & $15.1\pm0.77$
      & $75.5\pm1.09$
      & $13.1\pm0.69$
      & $79.0\pm0.88$ \\
    \textbf{ReAD}
      & $\mathbf{13.8\pm0.72}$
      & $\mathbf{78.1\pm0.97}$
      & $\mathbf{11.8\pm0.58}$
      & $\mathbf{81.6\pm0.83}$ \\
    \bottomrule
  \end{tabular}
\end{table}

\section{Related Work}
\label{sec:relate}

\textbf{Capability Distillation for Large Language Models.}\label{sec:relate_kd}
Recent work on capability distillation has focused on transferring specific capabilities such as instruction following and context manipulation~\cite{taori2023stanford, chiang2023vicuna, peng2023instruction, xu2023wizardlm}, thinking patterns~\cite{zhang2024knowledgeable, cheng2023adapting}, or text and code generation capability~\cite{zhang2025recommendation, li2023starcoder}. These methods typically steer a teacher LLM via prompts or rationales, generate distillation data or logits, and fine-tune the student on that evidence. However, the majority of methods operate under the assumption that distilling one target capability or domain is sufficient and rarely analyze how that focused training influences broader capability interactions within the student model. To the best of our knowledge, ReAD is the first framework that explicitly accounts for capability interdependence during the distillation.

\section{LLM Usage}
\label{app:llm_usage}

Large language models were used only for writing assistance, including grammar checking, wording refinement, formatting suggestions, and readability editing. They were not used to generate the core methodology, design experiments, produce experimental results, create evaluation labels, or make routing decisions. All technical claims, mathematical formulations, experimental settings, and reported results were checked and finalized by the authors.

\section{Broader Impact}
\label{sec:broader_impact}

ReAD aims to make capability distillation more budget-efficient by allocating
limited training tokens across interdependent LLM capabilities. A positive
impact is that smaller student models can become more useful under fixed compute
budgets, potentially reducing deployment cost and improving access to LLM-based
systems in resource-constrained settings. ReAD also explicitly monitors
cross-capability transfer, which can help practitioners diagnose when improving
one capability degrades others.

At the same time, more efficient distillation may lower the cost of producing
specialized models with dual-use capabilities, such as code generation, tool use,
or instruction following. Distilled models may also inherit biases, unsafe
behaviors, privacy risks, or other limitations from teacher models and training
data. In addition, if the capability probes or task requirement vectors are
misaligned with the intended deployment objective, ReAD may optimize benchmark
utility while overlooking safety-, fairness-, or robustness-relevant behaviors.

We therefore view ReAD as a research framework rather than a deployment-ready
safety guarantee. For sensitive applications, practitioners should combine ReAD
with task-specific safety evaluation, robustness and bias audits, data filtering,
red-team testing, and monitoring of off-target capability changes. Any release
or deployment of distilled models should follow the usage restrictions of the
underlying models and include appropriate documentation of intended use,
limitations, and safety evaluations.

\section{Limitation and Future Work}

While ReAD improves budget efficiency by using low-cost capability monitoring and adaptive allocation, it still relies on two practical components that can limit performance. First, the task requirement estimator and the probe suite may be imperfect: if the inferred essential capabilities or probe signals are misaligned with true task utility, ReAD can allocate budget suboptimally. Second, our capability set and data generator are necessarily finite and template-driven, which may not cover all behaviors a real task depends on, especially for highly domain-specific or long-horizon tasks. Future work can address these limitations by learning more robust and transferable task requirement models, designing probe suites with stronger coverage and better calibration to task utility, and expanding capability definitions and data generation beyond templates (e.g., automatic discovery of capability axes and open-ended curriculum generation). Another promising direction is to jointly optimize allocation and student training dynamics (e.g., optimizer schedules and parameter-efficient adapters) under the same budget, and to extend ReAD to settings with distribution shift, multi-task deployment, or continual updates where requirements and transfer patterns evolve over time.

\end{document}